\documentclass{article}
\usepackage[preprint]{neurips_2026}

\usepackage[utf8]{inputenc}
\usepackage[T1]{fontenc}
\usepackage[hidelinks]{hyperref}
\usepackage{url}
\usepackage{booktabs}
\usepackage{amsfonts,amsmath,amssymb,amsthm}
\usepackage{graphicx}
\usepackage{placeins}
\usepackage{float}
\graphicspath{{figs/}}
\usepackage{microtype}
\usepackage{xcolor}
\usepackage{pifont}

\newtheorem{theorem}{Theorem}
\newtheorem{proposition}{Proposition}
\newtheorem{lemma}{Lemma}

\theoremstyle{remark}
\newtheorem{remark}{Remark}

\newcommand{\Ex}{\mathbb{E}}
\newcommand{\Ot}{\widetilde{O}}
\newcommand{\etwo}{E_2}
\newcommand{\lamtwo}{\Lambda_2}
\newcommand{\Sset}{\mathcal{S}}
\newcommand{\cm}{\ding{51}}
\newcommand{\xm}{\ding{55}}

\workshoptitle{Second Workshop on ML$\times$OR: Mathematical Foundations and Operational
Integration of Machine Learning for Uncertainty-Aware Decision-Making}

\newcommand{\vhat}{\widehat{\bar v}}
\newcommand{\repourl}{https://github.com/melikabaghi/state-exp3}
\title{Pooling and Drift in Delayed Bandits}
\author{%
  Melika Baghi \\
  Georgia Institute of Technology \\
  \texttt{mbaghi3@gatech.edu}
}

\begin{document}
\maketitle

\begin{abstract}
A system often has to act long before it learns whether the act worked: a recommender sees a click
in seconds and a purchase in days. For $K$ actions and a delay of $d$ rounds, methods that see only the
action already attain the minimax-optimal dependence on the delay, and with an intermediate state
visible the best rate known is $\Ot(\sqrt{(K+d)T})$ over $T$ rounds.
Both are
governed by the number of actions, so a longer menu is always more expensive to learn from, and
neither adapts to instances in which many actions are effectively interchangeable. They need not
be. If the outcome depends on the action only through the state it produced, a click or a browse or
a cart, then one late outcome informs every action that could have produced the state observed, and
the price is set by how many genuinely different states the actions produce. We charge each delayed
outcome through the state and pay an \emph{effective dimension} $v_t$, between $1$ and the number of
states, in place of $K$. For the algorithm that is actually run we prove
$\Ot(\sqrt{V^-}+\sqrt{dT})$, where $V^-$ is any budget on $\sum_t(v_t-1)$ fixed in advance, so the
action count leaves the leading term and the delay enters additively. The gain has a limit: even when handed the exact losses from $d$ rounds ago, no algorithm escapes
$\Omega(\sqrt{d\mathcal E\min\{1+\log J,T/d\}})$, where $J$ counts the drifting directions and
$\mathcal E$ bounds how far the losses move while the learner waits. On generated data the state
channel cuts regret by up to $79$ per cent against action-level weighting, and on the funnel family
by $32$ to $68$ per cent against a rate-optimal action-only delayed-bandit
baseline under the tuning protocol we report. The
guarantees assume the action-to-state map is known, and every experiment is synthetic.
\end{abstract}

\section{Introduction}

Operations commit before they learn. A markdown is set and sell-through is known a season later;
staff are rostered and the shift's outcome lands after it; marketing spend is committed before
attribution resolves; an item is recommended and the purchase it may cause is days away. In each,
the decision is made now and the quantity it will be judged by arrives late, while something
cheaper arrives quickly: a fitting-room visit, a shift's first hour, a click. Whether that early
signal can cut the cost of learning is a question between two fields. What may be inferred from the
signal is a matter of data; the wait before the outcome resolves is a constraint on the operation,
which no estimator removes, and the decision commits while only the first of the two is settled.

The cost of learning under delay should be set by how many genuinely different regimes the actions
induce rather than by how many actions there are: a catalogue of
thousands of items that funnels users through half a dozen engagement patterns should be no harder
to learn from than a catalogue of six.

The mechanism behind that claim is that the outcome depends on the action only through the
intermediate state it produced, so one late outcome can be charged to every action that could have
produced the state observed rather than to the one action played, pooling it across them. Bandits with intermediate
observations formalize the setting: the learner picks one action and sees the consequence of that
action alone. \citet{esposito2023} show that the \emph{state-to-loss} map decides the complexity, an
unrestricted one giving only the rate available with no intermediate observations; we fix that map
and ask what the \emph{action-to-state} structure buys. The delayed-feedback literature has since
developed along other axes: \citet{baron2025} obtain instance-dependent bounds when the
observations themselves evolve, and \citet{levy2025} treat contextual actions under delay with
general function approximation. In neither does an intermediate state carry one delayed outcome
across several actions.

We answer that question with an \emph{effective dimension} $v_t$, measuring how distinguishable
the action-induced state distributions are under the current play; it lies between $1$ and the
number of states and can fall far below $K$. Theorem~\ref{thm:m1} pays $v_t$ in place of $K$ for the
algorithm that is actually run, and across the families the measured gain tracks $v_t$ rather
than $K$. The claim has a boundary. Waiting
carries a cost of its own, which sharing cannot remove. Writing
$m_t=c_{t-d}$ for the action-loss vector from $d$ rounds earlier and
$\etwo=\sum_t\|c_t-m_t\|_\infty^2$ for how far the losses move while the learner waits,
Theorem~\ref{thm:lower} lower-bounds the regret this forces even when $m_t$ is supplied free
(Figure~\ref{fig:one}). Adapting to overlap and drift at once remains open
(Section~\ref{sec:conj}). Four ingredients separate the neighbouring settings, and it is their
combination that is new here (Table~\ref{tab:settings}).

\begin{figure}[t]
\centering
\includegraphics[width=\textwidth]{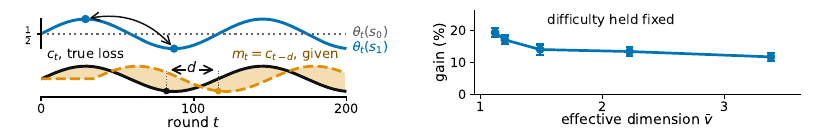}
\caption{\textbf{The claim and its boundary, on generated data.} \textbf{Left,} what
each state predicts changes while the action-to-state map stays fixed, so losses from $d$ rounds
earlier grow inaccurate; $\etwo$ of Section~\ref{sec:setting} tracks it. The marked points are the
true loss at $t-d$ and the same value handed to the learner at $t$, which the rule spans. \textbf{Right,} more
overlap allows more sharing, at matched difficulty, $K=40$, $|\Sset|=8$, $T=10000$, $d=50$, $20$
paired seeds, bars one standard error of the paired difference (Appendix~\ref{app:num}, Experiment~1).}
\label{fig:one}
\end{figure}

\begin{table}[H]\centering\small
\caption{Where this work sits: the four ingredients that separate the neighbouring settings.}
\label{tab:settings}
\begin{tabular}{@{}lcccc@{}}
\toprule
 & delayed & state seen & $P$ known & instance-dependent \\
\midrule
\citet{zimmert2020}   & \cm & \xm & \xm & \xm \\
\citet{esposito2023}  & \cm & \cm & \xm & \xm \\
\citet{eldowa2024}    & \xm & \cm & \cm & \cm \\
\citet{baron2025}     & \cm & \xm & \xm & \cm \\
this paper            & \cm & \cm & \cm & \cm \\
\bottomrule
\end{tabular}
\end{table}

\citet{eldowa2024} is closest, and $v_t-1$ is exactly the $\chi^2$ mutual information they
introduce, so that quantity is theirs; what is new is its role once the outcome is delayed. Other
neighbouring settings are placed in Appendix~\ref{app:transfer}.

\section{Problem formulation}\label{sec:setting}

The question is whether the intermediate state can lower the cost of learning below what $K$ alone
forces. The model that makes it precise is the following.
There are $T$ rounds, $K$ actions, a finite state space $\Sset$, a known action-to-state matrix
$P$, and a fixed known delay $d$; all
logarithms are natural and $\Ot$ hides logarithmic factors. Every round follows the chain
$A_t\xrightarrow{\;P\;}S_t\xrightarrow{\;\theta_t\;}X_t$: the learner draws $A_t$ from a distribution $x_t$,
observes $S_t\sim P(\cdot\mid A_t)$ at once, and receives the delayed outcome only at time $t+d$, encoded as a loss
$X_t\in[0,1]$ so that smaller is better, with
$\Ex[X_t\mid\mathcal F_{t-1},A_t,S_t]=\theta_t(S_t)$, where $\mathcal F_t$ is the environment's history
through round $t$, including outcomes generated but not yet delivered; outcomes are conditionally
independent across rounds given the states. The conditional mean depends on $A_t$ only through $S_t$. This
\emph{state-sufficiency} assumption licenses cross-action pooling: an outcome observed at state
$S_t$ may update every action that could have produced that state. If an action affects the outcome
directly, in a way not represented in $S_t$, the argument no longer applies. The learner may use all states observed so far, but
among outcomes only those from rounds $r\le t-d-1$ have been delivered. \textsc{State-EXP3} changes
its play distribution only when an outcome arrives, so for this algorithm $x_t$ is
$\mathcal F_{t-d-1}$-measurable.

The matrix $P\in[0,1]^{K\times|\Sset|}$ is row-stochastic and fixed, and the \emph{state-loss means} $\theta_t\in[0,1]^{|\Sset|}$ are \emph{oblivious}:
chosen in advance, not reacting to the learner. Action losses are $c_t=P\theta_t$ and performance
is the static \emph{pseudo-regret}, the gap to the best single action in hindsight,
$\Ex[R_T]=\Ex[\sum_tc_t(A_t)]-\min_a\sum_tc_t(a)$. \textbf{Oracle for the lower bound.} For Theorem~\ref{thm:lower} only, the learner also receives the \emph{stale action-loss vector} $m_t\triangleq c_{t-d}$ before choosing $A_t$, with $m_t=c_1$ for $t\le d$. This oracle information is unavailable to
\textsc{State-EXP3}, so a lower bound that holds even when it is supplied is stronger. We measure how misleading $m_t$ is by $\etwo=\sum_t\|c_t-m_t\|_\infty^2\le T$, which stays small when
errors are small even if frequent. The geometry of $P$ governs sharing across actions and the
evolution of $\theta_t$ governs staleness across time.

\section{Pooling through the state}\label{sec:alg}

These two sources of difficulty separate cleanly. Pooling replaces the leading action-count
dependence by an effective dimension, while no amount of sharing removes the cost of waiting.

An arriving outcome need not update the action played alone. Every action can be credited by how
likely it was to have produced the observed state. Here $P$ is known and $m_t$ is not given. For each state
$s\in\Sset$ let $q_t(s)=\sum_{a=1}^Kx_t(a)P(s\mid a)$, and give each action $a\in\{1,\dots,K\}$ the
estimate $\hat c_t(a)=P(S_t\mid a)X_t/q_t(S_t)$. It is correct on average, and its noise is governed by the
\emph{effective dimension}
$v_t=\sum_sq_t(s)^{-1}\sum_ax_t(a)P(s\mid a)^2\in[1,|\Sset|]$.
\textsc{State-EXP3} is \textsc{EXP3} \citep{auer2002} at learning rate $\eta$ on these estimates: it
adds $\hat c_t$ to its running total once the outcome is usable, at round $t+d+1$. A round costs
$O(K)$ (Appendix~\ref{app:pool}, with pseudocode and the $\chi^2$ reading of $v_t$).

\begin{theorem}\label{thm:m1}
Let $P$ be known and $(\theta_t)$ oblivious, and fix in advance any $V^-$ with
$\sum_t(v_t-1)\le V^-$ always.
Then \textup{\textsc{State-EXP3}} with any $\eta\le1/(e(d+1))$ satisfies
$\Ex[R_T]\le\sqrt{2(3.3V^-+2dT)\log K}+e(d+1)\log K+d$.
\end{theorem}

Thus the leading action-count dependence is replaced by $V^-$, which measures how differently the
actions behave rather than how many there are, while the delay enters additively at $\Ot(\sqrt{V^-}+\sqrt{dT})$, for the algorithm every
experiment below runs. Writing $\bar v$ for the largest $v_t$ over all ways of playing,
$V^-=(\bar v-1)T$ is admissible and depends on $P$ alone. An interleaved variant that runs
$d+1$ copies, and a version that merges similar states, are deferred to
Appendix~\ref{app:etc}; neither adapts to $\etwo$.

The proof uses one key identity. The denominator of $\hat c_t$ is exactly the probability of the
observed state under the current play, so $\langle x_t,\hat c_t\rangle=X_t\le1$ on every path. This
limits how far the play distribution can move while an outcome is in flight and controls the extra
term created by delay; subtracting $X_t\mathbf 1$ in the analysis replaces $v_t$ by $v_t-1$
(Appendix~\ref{app:m1}).

\paragraph{The cost of waiting.}\label{sec:lower} Waiting carries a cost that persists even when the
stale losses are supplied free. Theorem~\ref{thm:lower} isolates that cost from time alone: the
construction gives each of the $J$ drifting directions a private state, which removes useful
pooling and forces $J\le|\Sset|-1$ (Appendix~\ref{app:proofs}).

\begin{theorem}\label{thm:lower}
Fix $T$, a delay $1\le d\le T/2$, a drift budget $\mathcal E\le T/4$, and
$1\le J\le\min\{K-1,|\Sset|-1\}$. Every algorithm, even one handed the exact losses from $d$ rounds
ago, meets some instance fixed in advance with $\etwo\le\mathcal E$ on which
$\Ex[R_T]=\Omega(\sqrt{d\,\mathcal E\min\{1+\log J,\,T/d\}})$. Without that help, and if $K\le T$,
the best regret any algorithm can guarantee is
$\Omega(\sqrt{KT}+\sqrt{d\,\mathcal E\min\{1+\log J,\,T/d\}})$, up to constants.
\end{theorem}

The first term is the usual cost of exploring. The second is the cost of delay that pooling cannot
remove: it grows with the
wait, how far the losses move during it, and the number $J$ of directions the best fixed action can
use. They come from different hard instances, so adding them overstates the truth by at most a
factor of two.

\section{Experiments}

Theorem~\ref{thm:m1} predicts that the learning cost follows the number of distinct action
behaviours
rather than the number of actions, so the families below are built where the two come apart.
Figure~\ref{fig:phase} separates them in both directions: the gain rises as the effective dimension
falls at fixed task difficulty, and the effective dimension stays nearly constant as the catalogue
grows.
With $K=200$ generated items and six levels of engagement the estimated dimension is $1.97$, and
\textsc{State-EXP3} cuts regret against action-level weighting by $79.3$, $63.6$ and $38.9$
per cent at $d=10,50,200$, and against the rate-optimal action-only
delayed-bandit baseline of \citet{zimmert2020} tuned over a grid by $67.9$, $51.4$ and $31.5$. That setting is
hand-structured, not fitted to data, so it shows the mechanism, not a deployed system
(Table~\ref{tab:funnel}).

\begin{table}[H]\centering\small
\caption{Funnel family, $K=200$, $|\Sset|=6$, $T=20000$, $8$ paired seeds; cuts are against
action-level weighting and against the tuned baseline (Experiment~2).}
\label{tab:funnel}
\begin{tabular}{@{}rrrrrr@{}}
\toprule
$d$ & action-level & tuned \citep{zimmert2020} & \textsc{State-EXP3} & cut vs action & cut vs tuned \\
\midrule
 10 & 4577 & 2959 &  949 & $79.3\%$ & $67.9\%$ \\
 50 & 4752 & 3557 & 1728 & $63.6\%$ & $51.4\%$ \\
200 & 5307 & 4735 & 3245 & $38.9\%$ & $31.5\%$ \\
\bottomrule
\end{tabular}
\end{table}

These runs use the algorithm Theorem~\ref{thm:m1} covers. A best-state rule leads seven of nine settings, but where its assumption fails
its regret grows linearly in $T$ while ours stays sublinear
(Appendix~\ref{app:num}).

\begin{figure}[t]
\centering
\includegraphics[width=\textwidth]{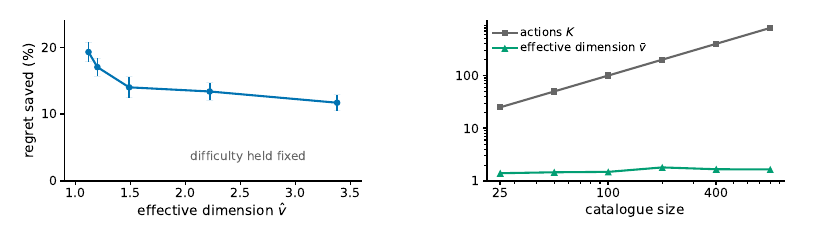}
\caption{\textbf{Effective dimension, not action count.} \textbf{Left,} regret saved against
action-level weighting as the effective dimension falls, with task difficulty held fixed so the
sweep is not confounded by easier instances; $K=40$, $|\Sset|=8$, $T=10000$, $d=50$, $20$ paired
seeds, bars one standard error of the paired difference (Experiment~1). \textbf{Right,} the catalogue
grows thirty-twofold while $\bar v$ stays between $1.4$ and $1.8$, the separation
Theorem~\ref{thm:m1} charges for (Experiment~2). $\vhat$ and $\bar v$ estimate $\sup_xv(x)$ by
Monte Carlo.}
\label{fig:phase}
\end{figure}

\section{Discussion}\label{sec:conj}

Two things stand between an action and its worth: how much its behaviour overlaps with other
actions, and how long the outcome takes to arrive. The first is exploitable: when $P$ is known, an
estimate of $\bar v$ well below $K$ indicates that delayed outcomes can be pooled across actions
and that the state channel is worth using. The second cannot be pooled away:
Theorem~\ref{thm:lower} shows that uncertainty from waiting persists even when the exact losses
from $d$ rounds ago are supplied for free. Thus, under state sufficiency,
intermediate states can remove the leading action-count cost of a large catalogue, but not the
temporal cost of stale information. Two limitations remain:
$\bar v$ is estimated rather than computed exactly, and our upper bound captures overlap while the
lower bound characterises drift on a separate axis. Adapting to both overlap and drift in one
algorithm and one guarantee remains open.

\appendix

\section{The state-pooled estimator}\label{app:pool}

\paragraph{The algorithm.} Inputs are the action-to-state matrix $P$, the delay $d$, the number
of copies $m$, and the rate $\eta>0$. Setting $m=1$ gives \textsc{State-EXP3}, the algorithm of the
body and of every experiment, covered by Theorem~\ref{thm:m1}; setting $m=d+1$ gives a distinct
algorithm, \textsc{Interleaved-State-EXP3}, which runs $d+1$ copies in rotation with copy
$i\in\{0,\dots,m-1\}$ taking every $m$th round, and is covered by Theorem~\ref{thm:sexp3}. Round
$r$'s outcome becomes usable at round $r+d+1$, which for $m=d+1$ is the next round belonging to the
copy that played round $r$.

\begin{enumerate}\itemsep1pt \parskip0pt \topsep3pt
\item Set $L_i\leftarrow(0,\dots,0)\in\mathbb R^K$ for $i=0,\dots,m-1$, and let $\mathcal P$ be
an empty table of rounds awaiting their outcome.
\item For $t=1,\dots,T$:
\begin{enumerate}\itemsep1pt \parskip0pt \topsep1pt
\item \emph{Deliver.} If $r\triangleq t-d-1\ge1$, read $(i_r,q_r(S_r),S_r,X_r)$ from
$\mathcal P$, remove it, and update only that copy,
$L_{i_r}(a)\leftarrow L_{i_r}(a)+P(S_r\mid a)X_r/q_r(S_r)$ for every $a$. This is the pooled
step, and it touches every action rather than the one played.
\item \emph{Select the copy.} $i\leftarrow t\bmod m$.
\item \emph{Play.} $x_t(a)\propto\exp(-\eta L_i(a))$, draw $A_t\sim x_t$, and observe $S_t$
at once.
\item \emph{Record.} Compute the scalar $q_t(S_t)=\sum_ax_t(a)P(S_t\mid a)$ and store
$(i,q_t(S_t),S_t,\cdot)$ in $\mathcal P$, to be completed when $X_t$ arrives at time $t+d$.
\end{enumerate}
\end{enumerate}

Steps (a), (c) and (d) are each $O(K)$ and there is no optimization subproblem, so a round costs
$O(K)$ once $P$ is stored. The state is held in the $m$ vectors $L_i$, which is $mK$ numbers, the
matrix $P$, which is $K|\Sset|$ numbers, and four scalars for each of the at most $d$ rounds in
flight. Only the scalar $q_r(S_r)$ is needed from round $r$, not the vector $x_r$.

\paragraph{A worked example.} Take three actions and two states, with
\[
P=\begin{pmatrix}0.8&0.2\\0.4&0.6\\0&1\end{pmatrix},
\qquad x_t=(\tfrac12,\tfrac14,\tfrac14),
\]
so $q_t(s_1)=\tfrac12(0.8)+\tfrac14(0.4)+\tfrac14(0)=0.5$ and $q_t(s_2)=0.5$. The third action
cannot produce $s_1$ at all. Suppose the learner draws $a_1$, observes $s_1$, and receives $X_t=1$
at time $t+d$. Then $\hat c_t(a)=P(s_1\mid a)/0.5$, so
$\hat c_t=(1.6,\,0.8,\,0)$ against the action-level $\tilde c_t=(2,0,0)$. One observation updates
all three actions: $a_2$ is charged although it was never played, because it had a $0.4$ chance of
producing the state seen, and $a_3$ is charged nothing, because it could not have produced that
state. None of this introduces bias: with $\theta_t=(1,0)$ the true losses are
$c_t=(0.8,0.4,0)$, and averaging $\hat c_t$ over $S_t\sim q_t$ returns exactly that. The gain is in the second moment:
$v_t=1.44$ here, against $K=3$ for the action-level estimate, and with $\theta_t=(1,1)$ both are
attained exactly.

\begin{remark}[The $\chi^2$ form]\label{rem:chi2}
Under $x_t$ the pair $(A_t,S_t)$ has law $x_t(a)P(s\mid a)$ with marginals $x_t$ and $q_t$, so
\begin{align*}
\chi^2\bigl(P_{A_t,S_t}\,\big\|\,P_{A_t}\!\otimes\!P_{S_t}\bigr)
&=\sum_{a,s}\frac{\bigl(x_t(a)P(s\mid a)-x_t(a)q_t(s)\bigr)^2}{x_t(a)q_t(s)}\\
&=\sum_{a,s}\frac{x_t(a)P(s\mid a)^2}{q_t(s)}-2+1\ =\ v_t-1,
\end{align*}
the cross term summing to $\sum_{a,s}x_t(a)P(s\mid a)=1$ and the last to
$\sum_{a,s}x_t(a)q_t(s)=1$. So $v_t=1$ exactly when $A_t$ and $S_t$ are independent, which pins the rows of $P$ only on the
support of $x_t$ and therefore means all rows agree whenever $x_t$ has full support, as
exponential weights does. If $P$ is deterministic and every state is reached by some action in
the support of $x_t$, then $v_t$ equals the number of states that support reaches. In particular,
under full-support play and an onto deterministic $P$, $v_t=|\Sset|$. This is an identity, not a bound; it is used for reading $v_t$, and once in the proof of
Theorem~\ref{thm:m1}.
\end{remark}

\begin{proposition}\label{prop:pool}
Fix $t$, let $x_t$ be the play distribution, $q_t(s)=\sum_ax_t(a)P(s\mid a)$ the induced state
distribution, and $\hat c_t(a)=P(S_t\mid a)X_t/q_t(S_t)$, which is well defined because only
states with $q_t(s)>0$ occur. For statements involving $1/x_t(a)$, restrict to actions with
$x_t(a)>0$; \textup{\textsc{State-EXP3}} itself has full support on every action. Then $\hat c_t(a)\ge0$, $\hat c_t(a)\le1/x_t(a)$,
$\Ex[\hat c_t(a)\mid\mathcal F_{t-1}]=c_t(a)$, and
\[
\sum_ax_t(a)\Ex[\hat c_t(a)^2\mid\mathcal F_{t-1}]\ \le\ v_t
:=\sum_{s:q_t(s)>0}\frac{\sum_ax_t(a)P(s\mid a)^2}{q_t(s)},\qquad 1\le v_t\le|\Sset| ,
\]
states of zero probability being omitted throughout, which is the convention $0/0=0$.
The action-level $\tilde c_t(a)=\mathbf 1\{A_t=a\}X_t/x_t(a)$ has the same first three properties
and weighted second moment at most $K$.
\end{proposition}

\begin{proof}
Nonnegativity is immediate from $X_t\ge0$. Given $\mathcal F_{t-1}$ the state $S_t$ has
distribution $q_t$, and $\Ex[X_t\mid\mathcal F_{t-1},S_t=s]=\theta_t(s)$ by the model, since that
conditional mean does not depend on the action, so
$\Ex[\hat c_t(a)\mid\mathcal F_{t-1}]=\sum_sq_t(s)P(s\mid a)\theta_t(s)/q_t(s)=c_t(a)$. For the
range, $q_t(s)\ge x_t(a)P(s\mid a)$ for every $a$, so $P(s\mid a)/q_t(s)\le1/x_t(a)$ and
$X_t\le1$. For the second moment, $X_t^2\le1$ gives
$\Ex[\hat c_t(a)^2\mid\mathcal F_{t-1}]\le\sum_sP(s\mid a)^2/q_t(s)$, and exchanging the sums
gives $v_t$. Each summand of $v_t$ is at most $1$, since $P(s\mid a)\le1$ makes
$\sum_ax_t(a)P(s\mid a)^2\le q_t(s)$, and at least $q_t(s)$, since Jensen applied to
$a\mapsto P(s\mid a)$ under $x_t$ gives $\sum_ax_t(a)P(s\mid a)^2\ge q_t(s)^2$. Summing the two
bounds gives $1\le v_t\le|\Sset|$. The action-level claims are the same computation with the
state replaced by the played action, whose weighted second moment is $\sum_ax_t(a)/x_t(a)=K$.
\end{proof}

The two ends are attained. If every action induces the same state distribution then
$v_t=\sum_sq_t(s)=1$, and if $P$ is deterministic, each action reaching a single state, then
every $P(s\mid a)$ is $0$ or $1$ and each summand is exactly $1$, so $v_t$ counts the states
reached and equals $|\Sset|$ when every state has an action reaching it. Disjoint row supports
give $v_t=K$ rather than $|\Sset|$, since the summands over one action's private states add to
$\sum_sP(s\mid a)=1$, so a sensor with many private states per action resolves few of them. So $v_t$ reads the
overlap between the rows of $P$ rather than their length, and $\bar v=\sup_xv(x)$, the supremum
over play distributions, depends on $P$ alone. An outcome observed at a state updates the estimate
for every action that can reach it, which is why the construction of Section~\ref{sec:lower} gives
each drifting coordinate a private state, forcing $v_t=|\Sset|$ there. The immediately observed
pairs $(A_t,S_t)$ identify $P$ separately, though Theorem~\ref{thm:sexp3} assumes it known.

\section{Proofs: pooling and grouping}\label{app:etc}

Two results are stated here rather than in the body, since Theorem~\ref{thm:m1} covers the
algorithm actually run. The first is the interleaved variant.

\begin{theorem}\label{thm:sexp3}
Let $P$ be known and $(\theta_t)$ oblivious. Pick before the run any fixed number $V$ with
$\sum_tv_t\le V$ always. Then \textup{\textsc{Interleaved-State-EXP3}} at the matching learning
rate satisfies $\Ex[R_T]\le\sqrt{2(d+1)V\log K}$. Writing $\bar v$ for the largest $v_t$ over all
ways of playing, $V=\bar vT\le|\Sset|T$ is one such choice, and it depends on $P$ alone.
\end{theorem}

The second merges states that are not worth telling apart. Group them with a map $g$ from $\Sset$
onto a smaller set $\mathcal Z$ and run the same algorithm on $g(S_t)$, where
$P^g(z\mid a)=\sum_{s:g(s)=z}P(s\mid a)$. Grouping never raises the dimension, but the estimate now
tracks a group average, so the widest spread of losses inside a group,
$\delta_t(g)=\max_z\max_{s,s':g(s)=g(s')=z}|\theta_t(s)-\theta_t(s')|$, is the error introduced.

\begin{theorem}\label{thm:rep}
Under the hypotheses of Theorem~\ref{thm:sexp3}, running on $g$ with $\sum_tv_t(g)\le V(g)$ gives
$\Ex[R_T]\le\sqrt{2(d+1)V(g)\log K}+2\sum_t\delta_t(g)$.
\end{theorem}

Merging need not preserve state-sufficiency, so the grouped estimate is right not for $c_t$ but for
a stand-in loss differing from it by at most $\delta_t(g)$ on every action. Keeping states apart
preserves differences that matter but learns more slowly; merging shares more but pays
$2\sum_t\delta_t(g)$. The best grouping is the true one: with $16$ states from four hidden groups,
sizes $1,2,4,8,16$ give regret $1590$, $831$, $754$, $920$, $1080$, lowest at four, which no
algorithm was told. The map $g$ is fixed in advance; learning it from data remains open.

\textbf{Proof idea.} The pooled estimator is unbiased for every action, while
Proposition~\ref{prop:pool} bounds its play-weighted second moment by $v_t$ rather than $K$. With
$m=d+1$ interleaved copies, the outcome generated on a round belonging to one copy is available
before that copy acts again, so each copy is an ordinary exponential-weights process with
immediate feedback. Summing the $d+1$ potential inequalities gives $(d+1)\log K/\eta$, and the
pooled second moments give $\tfrac\eta2\sum_tv_t$.

\begin{proof}[Proof of Theorem~\ref{thm:sexp3}]
Fix a copy $i$ and let $T_i$ be the number of rounds it owns, so $\sum_iT_i=T$. Round $t$'s
outcome is usable from round $t+d+1$, which with $m=d+1$ is exactly $t+m$, the copy's next round,
so within a copy the feedback is undelayed and $L_i$ carries exactly the estimates of that copy's
earlier rounds. Exponential weights on nonnegative losses give, for any action $a^\star$ and any
$\eta>0$,
\[
\sum_{t\in i}\bigl(\langle x_t,\hat c_t\rangle-\hat c_t(a^\star)\bigr)
\le\frac{\log K}{\eta}+\frac{\eta}{2}\sum_{t\in i}\sum_ax_t(a)\hat c_t(a)^2 ,
\]
by the usual potential telescoping with $e^{-z}\le1-z+z^2/2$, which needs $z\ge0$ and therefore
$\hat c_t\ge0$, supplied by Proposition~\ref{prop:pool}. Each $x_t$ is
$\mathcal F_{t-1}$-measurable, because $L_i$ then holds only rounds $r\le t-m=t-d-1$, so taking
expectations turns the left side into $\sum_{t\in i}(\langle x_t,c_t\rangle-c_t(a^\star))$ by
unbiasedness and the quadratic term into at most $\Ex[\sum_{t\in i}v_t]$, again by
Proposition~\ref{prop:pool}. Summing over the $m=d+1$ copies and using
$\sum_i\min_a\sum_{t\in i}c_t(a)\le\min_a\sum_tc_t(a)$, which is what lets the copies be compared
against one common action, gives $\Ex[R_T]\le(d+1)\log K/\eta+\tfrac\eta2\Ex[\sum_tv_t]$. If
$\sum_tv_t\le V$ surely then $\eta=\sqrt{2(d+1)\log K/V}$ balances the two terms and gives
$\sqrt{2(d+1)V\log K}$, and $V=\bar vT$ is admissible by the definition of $\bar v$.
\end{proof}

The delay enters as a factor because the copies do not share their estimates, each paying
$\log K/\eta$ over $T/(d+1)$ rounds. Sharing them gives the $m=1$ algorithm of
Theorem~\ref{thm:m1}, whose additive delay term avoids this penalty.

\begin{proof}[Proof of Theorem~\ref{thm:rep}]
Running the algorithm on $g$ is running it on the sensor $(\mathcal Z,P^g)$. The surrogate
losses below depend on $x_t$ and so are predictable rather than oblivious, but the proof of
Theorem~\ref{thm:sexp3} uses obliviousness nowhere, only that each $c^g_t$ is
$\mathcal F_{t-1}$-measurable given the play, so it applies unchanged and bounds regret measured
in the surrogate losses
$c^g_t(a)=\sum_zP^g(z\mid a)\bar\theta_t(z)$, where
$\bar\theta_t(z)=\sum_{s:g(s)=z}q_t(s)\theta_t(s)/q_t^g(z)$ is the group mean the grouped
estimator is unbiased for, by the same computation as in Proposition~\ref{prop:pool} with
$S_t$ replaced by $g(S_t)$. Fix $a$ and $z$. Both $\{q_t(s)/q_t^g(z)\}$ and
$\{P(s\mid a)/P^g(z\mid a)\}$ are probability vectors on $\{s:g(s)=z\}$, so their
$\theta_t$-averages lie in a common interval of length $\delta_t(g)$ and differ by at most
$\delta_t(g)$. Weighting by $P^g(z\mid a)$ and summing over $z$ gives
$|c^g_t(a)-c_t(a)|\le\delta_t(g)$ for every $a$, hence
$\langle x_t,c_t\rangle-c_t(a^\star)\le\langle x_t,c^g_t\rangle-c^g_t(a^\star)+2\delta_t(g)$.
Summing over $t$ adds $2\sum_t\delta_t(g)$. Merging leaves $V(g)$ no larger than a bound valid for
the raw states: it suffices to merge two states, the general case following by induction on the
merges. Writing $P_1,P_2$ for the two columns, $B_j=\sum_ax_t(a)P_j(a)$ and
$\lambda=B_1/(B_1+B_2)$, Cauchy--Schwarz gives
$(P_1+P_2)^2\le\lambda^{-1}P_1^2+(1-\lambda)^{-1}P_2^2$ pointwise in $a$, so after dividing by the
merged denominator the merged summand is at most the two it replaces.
\end{proof}

\section{Relation to other inaccuracy measures}\label{app:transfer}

\paragraph{Neighbouring settings.} \citet{vernade2020} vary the action-to-state map while fixing
the state-to-loss map, where we do the reverse. \citet{wang2026} and \citet{zhang2026impatient}
measure a prediction error as $\etwo$ does, without the delay. Feedback graphs
\citep{alon2015,esposito2022,gabbianelli2023} reveal other actions' losses outright, where here the
state's loss must still be estimated from the one outcome that arrives.

Write $\lamtwo=\sum_t\min\{1,\|c_t-m_t\|_2^2\}$ for the inaccuracy measure of \citet{baron2025}.
Unlike $\etwo$, which charges only the worst action in a round, $\lamtwo$ grows with how many
actions are predicted wrongly, and since $\|v\|_\infty\le\|v\|_2\le\sqrt J\|v\|_\infty$ for a
vector $v$ with $J$ nonzero coordinates,
\[
\etwo\ \le\ \lamtwo\ \le\ J\,\etwo,
\]
with both ends attained: the left when a single action is mispredicted, the right when all $J$ are
mispredicted equally. So a bound in $\etwo$ is never weaker and can be $J$ times stronger, which is
why Theorem~\ref{thm:lower} is stated in $\etwo$. The two measures agree exactly when the
prediction error is concentrated on one action.

\section{Proofs: the lower bound}\label{app:proofs}

The measure $\etwo$ is controlled by the state-loss variation $W$, and no converse bound
holds in general.

\begin{lemma}[Delay window]\label{lem:window}
$\etwo\le d^2W$ for $W=\sum_t\|\theta_t-\theta_{t-1}\|_\infty^2$. The ratio $\etwo/(d^2W)$ tends to one when $\theta$ drifts monotonically in one coordinate at a
constant rate and $T/d\to\infty$, the first
$d$ rounds contributing $\sum_{j<d}j^2$ rather than $d^2$ each under the convention $m_t=c_1$.
\end{lemma}

In words, accumulated squared prediction error is at most the state-loss variation inflated by
the square of the delay, and no better bound in $W$ alone is available.

\begin{proof}[Proof of Lemma~\ref{lem:window}]
Throughout write $\Delta_j=\theta_j-\theta_{j-1}$, set $\theta_0=\theta_1$, so $\Delta_1=0$, and read
$c_r$ as $c_1$ for $r\le0$, matching the convention $m_t=c_1$ for $t\le d$ of
Section~\ref{sec:setting}. Then $c_t-c_{t-d}=P\sum_{j=t-d+1}^t\Delta_j$. Because the
rows of $P$ are probability vectors, averaging cannot increase the supremum norm, so
$\|c_t-c_{t-d}\|_\infty\le\sum_j\|\Delta_j\|_\infty$. By Cauchy--Schwarz applied to the $d$
summands, $\|c_t-c_{t-d}\|_\infty^2\le d\sum_{j=t-d+1}^t\|\Delta_j\|_\infty^2$. Summing over $t$
and using that each increment belongs to at most $d$ windows gives $\etwo\le d^2W$. The ratio $\etwo/(d^2W)$ approaches one when every increment has the same magnitude, sign and
active coordinate, and some action reaches that coordinate deterministically, since then no
cancellation occurs within a full window and $P$ passes the increment through undiminished.
Without the second condition the drift can lie in the kernel of $P$ and the ratio can be zero.
\end{proof}

The construction in full. Take states $s_0,\dots,s_J$, a stationary deterministic map sending
$a_j$ to $s_j$ for $j\le J$ and every other action to $s_0$, and independent Rademacher signs
$\sigma_b^{(j)}$ for $2\le b\le B=\lfloor T/h\rfloor$. Set $\theta_t(s_0)=\tfrac12$ throughout,
$\theta_t\equiv\tfrac12$ on block~$1$, and $\theta_t(s_j)=\tfrac12-\varepsilon\sigma_b^{(j)}$ on
each later block $b$. The neutral first block makes $c_1=\tfrac12\mathbf 1$, so the
convention $m_t=c_1$ for $t\le d$ of Section~\ref{sec:setting} supplies a constant over exactly
the rounds that precede any usable feedback, which is what the isolation step below needs at $b=1$. Rounds after the last complete block carry no drift, costing a constant factor.

\begin{lemma}[Rademacher lower tail]\label{lem:rade}
Let $S_B$ be a sum of $B\ge32$ independent Rademacher signs. For every $1\le u\le\sqrt B/32$,
$\Pr(S_B\ge u\sqrt B)\ge\tfrac u4e^{-64u^2}$.
\end{lemma}

\begin{proof}
Write $S_B=2X-B$ with $X\sim\mathrm{Bin}(B,\tfrac12)$, so the event is
$\{X\ge B/2+u\sqrt B/2\}$; passing to $X$ removes the lattice gap, since $X$ takes every integer
value in $[0,B]$ whereas $S_B$ takes only those of the parity of $B$. Put $n=\lfloor B/2\rfloor$
and $t=\lceil u\sqrt B/2\rceil+1$, so $X\ge n+t$ implies the event and $t\le u\sqrt B$.
Consecutive binomial probabilities have ratio $(B-n-j+1)/(n+j)$, whose distance from one is at
most $8j/B$; for $j\le B/16$ that is at most $\tfrac12$, where $\log(1-x)\ge-2x$ holds, so summing
over $i\le j$ gives $\Pr(X=n+j)\ge\Pr(X=n)e^{-16j^2/B}$ with $\Pr(X=n)\ge1/(2\sqrt B)$. Summing
over the $t$ integers $j\in\{t,\dots,2t-1\}$, all at most $B/16$ because $u\le\sqrt B/32$, and
using $t\ge u\sqrt B/2$, gives
$\Pr(X\ge n+t)\ge t\,e^{-64u^2}/(2\sqrt B)\ge\tfrac u4e^{-64u^2}$.
\end{proof}

Three regimes give the maximal inequality the lower bound needs. Each $S_B^{(j)}$ is sub-Gaussian
with variance proxy $B$, so $\Ex\max_j(S_B^{(j)})^{+}\le\sqrt{2B\log J}+\sqrt B$ throughout. For
the matching lower bound, first suppose $\log J\le B/8$ and $\log J\ge128$. Taking
$u=\sqrt{\log J/128}$, which then lies in $[1,\sqrt B/32]$, makes $e^{-64u^2}=J^{-1/2}$, so
$\Pr(S_B^{(j)}\ge u\sqrt B)\ge\tfrac u4J^{-1/2}$ and independence across $j$ gives
$\Pr(\max_jS_B^{(j)}<u\sqrt B)\le e^{-u\sqrt J/4}$, at most $\tfrac12$ beyond an absolute $J$.
Hence $\Ex\max_j(S_B^{(j)})^{+}\ge\tfrac12u\sqrt B=\Omega(\sqrt{B\log J})$. Second, if
$\log J>B/8$ and $B\ge1024$, take $u=\sqrt B/32$, so $u\sqrt B=B/32$ and $e^{-64u^2}=e^{-B/16}$,
and $J\ge e^{B/8}$ makes $J\,e^{-B/16}$ diverge, giving
$\Ex\max_j(S_B^{(j)})^{+}=\Omega(B)$, which is $\Omega(\sqrt{B\min\{1+\log J,B\}})$ in that
regime. Third, when $\log J<128$ or $B<1024$ the quantity $\sqrt{B\min\{1+\log J,B\}}$ is
$O(\sqrt B)$, and a single walk already gives
$\Ex(S_B^{(1)})^{+}=\tfrac12\Ex|S_B|\ge\tfrac12(\Ex S_B^2)^{3/2}(\Ex S_B^4)^{-1/2}\ge\tfrac12\sqrt{B/3}$
by H\"older, since $\Ex S_B^2=B$ and $\Ex S_B^4\le3B^2$. Combining the three regimes,
$\Ex\max_j(S_B^{(j)})^{+}=\Theta(\sqrt{B\min\{1+\log J,B\}})$. This is the route by which maxima
of independent random walks fix the minimax rate for prediction with expert advice
\citep{cesabianchilugosi2006}.

\textbf{Proof idea.} Divide time into blocks of length $d$ and give $J$ actions private states
whose losses are perturbed by independent Rademacher signs within each block. Every observation
available when an action is chosen comes from an earlier block, and the supplied $m_t=c_{t-d}$
depends only on earlier signs, so the learner cannot predict the signs governing its current
block. Its expected loss is therefore $T/2$, while the best fixed action in hindsight benefits
from the maximum of $J$ independent random walks. Choosing the amplitude to make $\etwo\le\mathcal E$
gives the stated scale.

\begin{proof}[Proof of Theorem~\ref{thm:lower}]
Take $h=d$ and $\varepsilon=\tfrac12\sqrt{\mathcal E/T}$. Every outcome usable when choosing $A_t$
comes from a round $r\le t-d-1$, and $t\le bh$ with $h\le d$ forces $r<(b-1)h$, so $A_t$ is
independent of block $b$'s own signs; the same computation places the stale vector $m_t$ in blocks
strictly before $b$, so handing it to the learner changes nothing. Every algorithm's
expected loss therefore equals $T/2$, since its action is independent of the current block's signs and all states have mean
$\tfrac12$ marginally. Its regret is therefore exactly the comparator's
fluctuation advantage,
$\varepsilon d\,\Ex\bigl[\max_{j\le J}\bigl(\sum_{2\le b\le B}\sigma_b^{(j)}\bigr)^{+}\bigr]$ with
$B=\lfloor T/d\rfloor$, over the $B-1$ randomized blocks. Per round $\|c_t-m_t\|_\infty\le2\varepsilon$, because the supremum norm charges only the largest
coordinate and each coordinate moves by at most $2\varepsilon$ at a block boundary; the gap is
$\varepsilon$ across the first randomized block, where $m_t$ is still the neutral $c_1$, and lies
in $\{0,2\varepsilon\}$ thereafter. Hence $\etwo\le4\varepsilon^2T=\mathcal E$ holds
deterministically. The
budget must bind on the realized sequence, since a budget holding only in expectation would not
constrain the instance selected by Yao's principle. By Lemma~\ref{lem:rade} and the discussion following it, the expected maximum of the positive
parts of $J$ independent Rademacher sums of length $B-1$ is
$\Theta(\sqrt{(B-1)\min\{1+\log J,B-1\}})$, which covers both regimes and the case $J=1$. The
hypothesis $d\le T/2$ gives $B\ge2$ and hence $B-1\ge B/2$, so replacing $B-1$ by $B$ costs a
factor at most $\sqrt2$ in each branch of the minimum, including the saturated branch where the
minimum is $B-1$ rather than $1+\log J$. Substituting $B=\lfloor T/d\rfloor$ and $\varepsilon$
therefore gives $\Theta(\sqrt{d\mathcal E\min\{1+\log J,T/d\}})$. Rounds after the last complete block
carry no drift and change the constant only. When the oracle is withheld, the $\sqrt{KT}$ term is \citet[Thm.~5.1]{esposito2023}, whose
construction is stationary and therefore lies in the present model class.
\end{proof}

\begin{remark}[Where the bound is matched]\label{rem:tight}
Suppose $X_t=c_t(A_t)$. On instances whose drift is carried by a single action only one coordinate
of $c_t-m_t$ is nonzero, so $\|c_t-m_t\|_2=\|c_t-m_t\|_\infty$; on the family of
Section~\ref{sec:lower} that coordinate has magnitude at most $2\varepsilon\le1$, so the truncation
at one in $\lamtwo$ is inactive and $\lamtwo=\etwo$. The bound of \citet{baron2025} then reads
$\Ot(\sqrt{KT+d\,\etwo})$, whose $\sqrt{d\,\etwo}$ contribution Theorem~\ref{thm:lower} matches up
to the logarithmic factor at $J=1$, so on that family the drift term of the published upper bound
is unimprovable in order. Its $\sqrt{KT}$ term is not matched, the oracle supplying an entire stale
loss vector, so the oracle-assisted problem need not retain bandit exploration complexity. Neither
result provides an $\etwo$-adaptive upper bound under noisy outcomes.
\end{remark}

\section{Numerical detail}\label{app:num}

Sixteen studies were run; the three that bear most directly on the main claim and its boundary are
reported here as Experiments~1 to~3, each with an observation, an interpretation and a limitation,
and each named with the script that produces it. The code for all sixteen is available.
Experiment~1 draws $P$ from a Dirichlet; Experiment~2 uses a hand-structured funnel map; and
Experiment~3 uses the deterministic hard-instance map of Theorem~\ref{thm:lower}. All policies run on the same
environment parameters and common random-number streams, so a seed fixes everything except which
estimator is formed, and differences are taken within a seed before averaging. Round $r$'s outcome
is consumed before round $r+d+1$, matching Section~\ref{sec:setting}. Reported $\pm$ is one
standard error of the paired difference. \emph{Sample-path pseudo-regret} means
$\sum_tc_t(A_t)-\min_a\sum_tc_t(a)$ on one realization of the learner and environment randomness,
computed from the true losses rather than the sampled outcomes; averaging it over seeds estimates
the expected pseudo-regret the theorems bound.

The first two experiments probe the state channel of the claim, the first under tight control and
the second in an operational shape, and the last probes its temporal boundary. Experiment~1 holds
$K$ and the task difficulty fixed and varies the overlap, which separates
Theorem~\ref{thm:m1}'s prediction from the instances simply becoming easier; Experiment~2 shows the
regime in an environment built to the shape of a funnel and measures it against a rate-optimal
baseline. Experiment~3 then tests Theorem~\ref{thm:lower} on the drift channel, where pooling
cannot help. Theorems~\ref{thm:sexp3} and~\ref{thm:rep}, on the interleaved algorithm and on
coarsening, have no experiment here; both are stated and proved in Appendix~\ref{app:etc} and left
untested in this submission. Other studies whose conclusions were
superseded by a later control, or that replicated one already reported, are left to the released
code.

\begin{figure}[t]
\centering
\includegraphics[width=\textwidth]{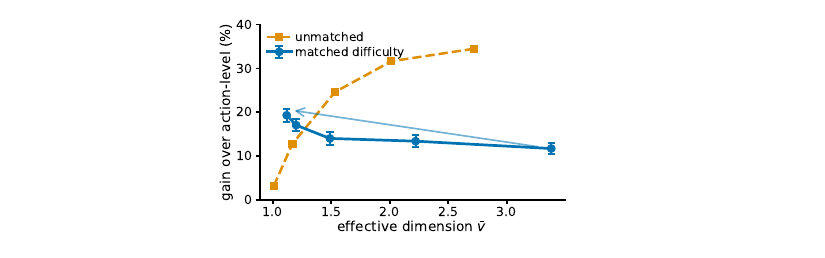}
\caption{Experiment~1. The same overlap sweep with and without holding the task
difficulty fixed. Matching the uniform-play regret reverses the trend, supporting the interpretation that the
decline in the unmatched sweep is driven by task-difficulty confounding rather than by the pooling
mechanism itself. Bars are one standard error of the paired
difference.}
\label{fig:matched}
\end{figure}

\FloatBarrier

\textbf{Experiment 1 (\texttt{matched\_control.py}). Overlap drives the gain, not easier tasks.}
Sweeping the overlap without further control finds the advantage falling as $\vhat$ falls, which
read at face value contradicts Theorem~\ref{thm:m1}, whose budget $V^-$ falls with $\bar v$.
That sweep is confounded, since raising the overlap also flattens $c_t=P\theta_t$ across actions, collapsing
the comparator gap so that both regrets vanish together. This experiment holds $K=40$, $|\Sset|=8$,
$T=10000$, $d=50$ and the difficulty fixed, and varies only $\vhat$. Difficulty is the
uniform-play regret $\sum_t\mathrm{mean}_ac_t(a)-\min_a\sum_tc_t(a)$, a property of the instance
with no algorithm in it, and the amplitude of $\theta$ is bisected in every cell to hit a common
target. The map keeps one contrast direction alive at every overlap and $\theta$ is aligned with
it, which is what lets the gap survive as the rows agree.

Observed. Across $20$ paired seeds the relative gain rises from $11.7$ per cent at
$\vhat=3.380$ to $19.3$ per cent at $\vhat=1.120$, with absolute gains $85.5\pm8.9$ to
$106.2\pm7.9$ and a correlation of $-0.929$ between $\vhat$ and the gain. The unmatched sweep
falls from $34.5$ to $3.1$ per cent over a comparable range (Figure~\ref{fig:matched}), so the
apparent contradiction is an artefact of the confound rather than of the mechanism.

Interpretation. The two sweeps disagree in sign, and the matched-difficulty sweep agrees with the overlap
dependence Theorem~\ref{thm:m1} predicts for the single-copy algorithm this experiment runs. Overlap is what pooling exploits,
and the unmatched sweep measures overlap confounded with an instance becoming trivial.

Limitation. At the two highest overlaps some seeds cannot reach the common target and the
realized difficulty is $5$ to $13$ per cent below it, so the matching is close rather than exact.
Since a smaller gap should if anything shrink the advantage, the measured rise is conservative in
that direction. The matched gains are also smaller in magnitude than the unmatched ones, which is
what fixing the difficulty costs.

\begin{figure}[t]
\centering
\includegraphics[width=\textwidth]{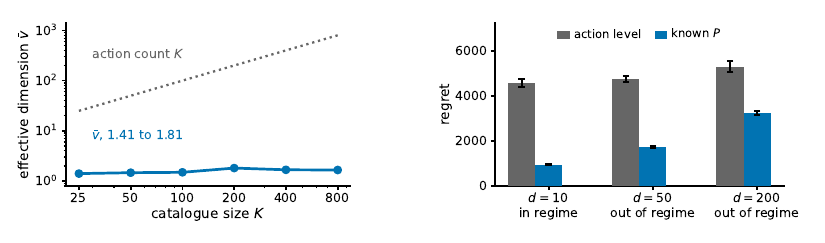}
\caption{Experiment~2, a recommendation funnel. \textbf{Left,} the effective dimension stays near two
while the catalogue grows thirty-twofold, so the gap between $K$ and $\vhat$ comes from the
catalogue having categories, not from a chosen $K$. \textbf{Right,} the two arms at $T=20000$
and $8$ paired seeds, labelled by whether $\bar v(d+1)\log K\le K+d$ holds.
Bars are one standard error. No real data is used anywhere in this experiment.}
\label{fig:funnel}
\end{figure}

\FloatBarrier

\textbf{Experiment 2 (\texttt{funnel.py}, \texttt{sota\_check.py}). A recommendation funnel, and a rate-optimal baseline on it.} Unlike Experiment~1, this
experiment does not draw $P$ from a Dirichlet; the Dirichlet construction is the assumption under
test rather than evidence for it. This
one builds an environment to the shape of a recommendation funnel instead. Actions are $K=200$
catalogue items, states are $|\Sset|=6$ ordered engagement depths, and the outcome is a delayed
conversion. Three structural facts are imposed because funnels have them, namely that items fall
into a modest number of categories sharing an engagement profile, that item quality is
heavy-tailed, and that the loss falls with engagement depth and drifts seasonally rather than
adversarially. The generated catalogue puts $74$ per cent of its mass on no engagement and $1.4$
per cent on the deepest. The same environments, seeds and delay convention then carry a comparison
against the hybrid Tsallis-plus-entropy \textsc{FTRL} of \citet{zimmert2020}, which attains the
optimal rate for delayed feedback, its delay term being additive rather than multiplicative. The
algorithms published since match that adversarial rate rather than improve on it, adding
adaptivity to stochastic instances instead \citep{schlisselberg2025}; all of them read the action
alone.

Observed. The effective dimension is $\vhat=1.97$ against $K=200$, a hundredfold gap, and it stays
between $1.41$ and $1.81$ as the catalogue grows from $25$ to $800$ items
(Figure~\ref{fig:funnel}). Against action-level weighting the state-pooled variant saves $79.3$,
$63.6$ and $38.9$ per cent at $d=10$, $50$ and $200$, the largest margins in this appendix.

\begin{table}[H]\centering\small
\caption{Regret on the funnel under three tunings (Experiment~2); grid-best columns are chosen
by reading these benchmarks.}
\label{tab:tuning}
\begin{tabular}{@{}lrrrr@{}}
\toprule
family, cell & \textsc{State-EXP3} & \citeauthor{zimmert2020} & same, grid-best & \textsc{State-EXP3}, grid-best \\
\midrule
funnel, $d=10$   &  949.0 & 2959.4 & 2413.8 & 27.0 \\
funnel, $d=50$   & 1728.1 & 3557.0 & 2556.4 & 95.9 \\
funnel, $d=200$  & 3244.8 & 4734.7 & 2807.7 & 199.7 \\
\bottomrule
\end{tabular}
\end{table}

Table~\ref{tab:tuning} sets the two against each other. With \textsc{State-EXP3} at the single
tuning used everywhere else here and \citet{zimmert2020} at the best of a six-point grid around its own value,
\textsc{State-EXP3} wins all three funnel cells, cutting regret by $31.5$ to $67.9$ per cent. With
both at their grid-best it wins all three again and by more. When only \citet{zimmert2020} is
grid-tuned, \textsc{State-EXP3} wins at $d=10$ and $d=50$ and loses at $d=200$.

Interpretation. The separation between $K$ and $\vhat$ arrives from catalogue structure rather than
from a tuned parameter, and it does not close as the catalogue grows, which is the claim the
Dirichlet families cannot make. Where the state-pooled variant is not held to the worse tuning of
the two, the state channel helps against a rate-optimal method, which is expected, since that
method cannot read the state at all; grid tuning helps it far more than it helps
\citet{zimmert2020}, consistent with pooling cutting estimator variance and lower variance
admitting a larger learning rate. The condition under which interleaving's bound would improve on
the action-level rate fails at $d=50$ and $d=200$, while the single-copy variant still saves most
of the regret there, which is the regime Theorem~\ref{thm:m1} was proved for.

Limitation. \emph{No real data is used.} There is no interaction log in this environment, nothing
here is fitted to one, and this is not a semi-synthetic benchmark in the usual sense of that
phrase. The structure is imposed by hand from qualitative properties, so it shows that the regime
is describable and self-consistent, not that any deployed system occupies it. The funnel cell at
$d=200$ lost when \citet{zimmert2020} alone is grid-tuned is a real loss and reported as one. No column is
a symmetric protocol: the two grids are finite and unequal, and \textsc{State-EXP3}'s grid-best
sits at its top multiplier in every cell, so its own minimum may lie outside the range searched.
Every grid-best multiplier is chosen by reading these benchmarks, so only the first column is a
procedure a practitioner could run. Eight paired seeds on the funnel.
The arms run $m=1$, covered by Theorem~\ref{thm:m1}, and the environment has no adversarial
component, so it exercises the pooling channel and not the drift channel.

\textbf{A heuristic baseline.} The best-state rule identifies the state with the lowest mean loss
and plays the action most likely to reach it, ignoring every other state. Across nine settings,
five families at $d=10$ and $50$ plus one drift cell, it leads seven and \textsc{State-EXP3} two.
Its two losses are the spread family, built so that the best action mixes over several good states
while a decoy action puts more mass than any other on the single best one: the rule takes the decoy
and records $1496\pm329$ and $1495\pm327$ at $d=10$ and $50$ against our $486\pm11$ and
$814\pm20$. Sweeping $T$ from $2500$ to $40000$ on that family and fitting regret to a power of $T$
gives an exponent of $1.00$ for the rule, $0.58$ for \textsc{State-EXP3} and $0.67$ for
action-level weighting, so the rule's regret grows linearly where both learners' grow sublinearly.
Its spread across seeds averages $4.6$ times ours over the nine settings. The rule needs one state
to be both best and reachable, and nothing in it detects when that fails, which is why the gain it
shows on seven settings is not a gain a practitioner can rely on.

\begin{figure}[t]
\centering
\includegraphics[width=\textwidth]{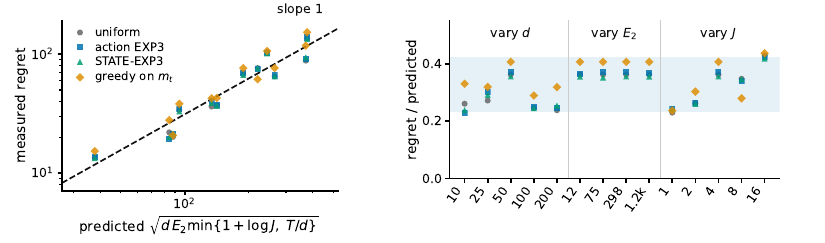}
\caption{Experiment~3, the drift channel. \textbf{Left,} measured regret against the scale
Theorem~\ref{thm:lower} predicts, over twelve cells in which $d$ varies $20$-fold, $\etwo$
$100$-fold and $J$ $16$-fold. The dashed line has slope one. \textbf{Right,} the same points
divided by that scale, grouped by which parameter the cell sweeps, with the shaded band spanning
the observed range. \textsc{State-EXP3} appears here only as one more algorithm on the hard
family, where it has no advantage.}
\label{fig:drift}
\end{figure}

\FloatBarrier

\textbf{Experiment 3 (\texttt{drift\_scaling.py}). The drift bound describes the right scaling.} The construction of
Theorem~\ref{thm:lower} at $K=40$, $T=8000$, $h=d$ and $12$ paired seeds, sweeping $d$ from $10$
to $200$, the amplitude $\varepsilon$ from $0.02$ to $0.20$ so that $\etwo$ runs from $12$ to
$1193$, and the drifting coordinates $J$ from $1$ to $16$. We run four policies: uniform, action-level \textsc{EXP3}, \textsc{State-EXP3},
and the greedy policy on the exact stale vector $m_t$, the last because
Theorem~\ref{thm:lower} claims to survive it. $\etwo$ is measured on the realized instance rather
than predicted, so the normalization uses the quantity the theorem uses.

Observed. The predicted scale $\sqrt{d\,\etwo\min\{1+\log J,T/d\}}$ ranges from $38$ to $377$
across the grid and the measured regret tracks it (Figure~\ref{fig:drift}). Normalized, the four
policies read $0.230$ to $0.429$, $0.228$ to $0.424$, $0.238$ to $0.417$ and $0.237$ to $0.437$,
a spread below $1.9$ against a tenfold range in the predictor. The oracle-assisted policy has the
largest mean ratio, $0.345$ against $0.311$, $0.314$ and $0.310$ for the three that never see
$m_t$.

Interpretation. The three arguments enter the predicted scale differently and the collapse holds
in all three sweeps, so it is the form of the bound rather than one fitted constant that tracks
the data. Greedy play on the exact $m_t$ sits at the top of the band rather than below it, which is what
the bound surviving the oracle amounts to in these runs.

Limitation. Every cell has $T/d>1+\log J$, so only the $1+\log J$ branch of the minimum is
exercised and the $T/d$ saturation branch, which needs $d$ comparable to $T$, is not probed. A
lower bound quantifies over all algorithms and four policies cannot establish that. What these
runs check is the scaling of the construction.

\textbf{Code and environment.} The code producing every reported number and figure is at
\expandafter\url\expandafter{\repourl}. It pins the software environment used for the reported
numbers, python 3.14.3 with numpy 2.4.2 and matplotlib 3.10.8, states the seed count for each
study, lists one command per numbered study, and commits the stored outputs each figure reads so
that every figure rebuilds without rerunning an experiment.

Taken together these experiments cover both channels. Experiments~1 and~2 probe the state channel
and are consistent with the effective dimension, rather than $|\Sset|$ or $K$, governing what
pooling buys; Experiment~1 had to control for task difficulty before an overlap sweep pointed that
way.
Among the state-channel experiments, the funnel is the only environment whose matrix is not
drawn from a Dirichlet and it reports the
largest margins here, on a structure imposed by hand rather than measured. It also places
\textsc{State-EXP3} against a rate-optimal delayed-bandit method: ahead in every cell when both are
tuned alike, behind at $d=200$ when that method alone is grid-tuned. Experiment~3 turns to the other
channel and finds the drift scale of Theorem~\ref{thm:lower} tracking the data, including for a
learner handed the stale losses, which is the cost no amount of pooling reaches.

\section{Proof of Theorem~\ref{thm:m1}: additive delay without interleaving}\label{app:m1}

Interleaving cannot give an additive bound. Copy $i$, which absorbs a share $V_i$ of the budget
$V=\sum_iV_i$, has regret $\sqrt{2V_i\log K}$, and
Cauchy--Schwarz over the $m$ copies gives $\sqrt{2mV\log K}$, tight when the $V_i$ agree, so the
$\sqrt m$ in Theorem~\ref{thm:sexp3} is not an artifact of the analysis. An additive bound has to
come from a single copy.

Throughout, $L_t=\sum_{r\le t-d-1}\hat c_r$ and $x_t\propto e^{-\eta L_t}$, so
$L_{t+1}-L_t=\hat c_{t-d}$, with $\hat c_r=0$ for $r<1$. The filtration $\mathcal F_t$ is the
environment's, as defined in Section~\ref{sec:setting}; under it
$x_t$ is $\mathcal F_{t-d-1}$-measurable, so $x_{r+d}$ is $\mathcal F_{r-1}$-measurable.

\paragraph{The drift lemma.} The natural route prices the change of measure from $x_r$ to
$x_{r+d}$ through the normaliser $Z_r$, whose reciprocal has no uniformly bounded moment. That is
avoidable. The pooled estimate satisfies
\[
\langle x_t,\hat c_t\rangle=\sum_ax_t(a)\frac{P(S_t\mid a)X_t}{q_t(S_t)}=X_t\le1
\]
identically, on every path, because the denominator is the state marginal of the same distribution
that supplies the numerator.

\begin{lemma}\label{lem:drift}
If $\eta\le1/(e(d+1))$ then $x_{t+1}(a)\le(1+1/d)\,x_t(a)$ for every $t$ and every $a$, surely, and
hence $x_{t+d}(a)\le e\,x_t(a)$.
\end{lemma}

\begin{proof}
Strong induction on $t$. For $t\le d$ nothing is delivered and $x_{t+1}=x_t$. For $t>d$, assume the
claim for every $s<t$ and compose it over the $d$ steps from $t-d$ to $t$, giving
$x_t(a)\le(1+1/d)^dx_{t-d}(a)\le e\,x_{t-d}(a)$. Contracting with the fixed nonnegative vector
$\hat c_{t-d}$ and using the identity above,
$\langle x_t,\hat c_{t-d}\rangle\le e\langle x_{t-d},\hat c_{t-d}\rangle=eX_{t-d}\le e$. With
$Z_t=\sum_bx_t(b)e^{-\eta\hat c_{t-d}(b)}$ and $e^{-z}\ge1-z$, this gives
$Z_t\ge1-\eta\langle x_t,\hat c_{t-d}\rangle\ge1-\eta e\ge d/(d+1)$, so
$x_{t+1}(a)=x_t(a)e^{-\eta\hat c_{t-d}(a)}/Z_t\le x_t(a)/Z_t\le(1+1/d)x_t(a)$.
\end{proof}

What the induction bounds is
$\langle x_t,\hat c_{t-d}\rangle=X_{t-d}\,q_t(S_{t-d})/q_{t-d}(S_{t-d})$, the state-marginal ratio that the
$1/Z_r$ route could not reach. Individual $\hat c_r(a)$ remain unbounded; only their
$x_t$-weighted mass is controlled. \citet{cesabianchi2019} prove the action-level version under
$\eta\le1/(Ke(d+1))$ and \citet{thune2019} remove the $K$, using that their estimate is carried by
the single action played, so the weight cancels into a same-action ratio. A pooled estimate is
spread over every action and does not collapse that way; the identity above replaces the
cancellation.

\paragraph{Centring.} Let $\zeta_t=\hat c_t-X_t\mathbf 1$, with $\zeta_r=0$ for $r<1$. Exponential
weights is unchanged by a shift that is constant across actions, so $x_t$ is the same iterate and
this is a change of analysis, not of algorithm \citep{eldowa2024}. Keeping
$\gamma(s)=\Ex[X_t^2\mid S_t=s]$ inside the state sum,
\[
\Ex\bigl[\langle x_t,\zeta_t^2\rangle\mid\mathcal F_{t-1}\bigr]
=\sum_s\gamma(s)\Bigl[\frac{\sum_ax_t(a)P(s\mid a)^2}{q_t(s)}-q_t(s)\Bigr]\le v_t-1,
\]
each bracket being nonnegative by Jensen and $\gamma(s)\le1$, the final step being
Remark~\ref{rem:chi2}. Bounding $\langle x_t,\hat c_t^2\rangle$ by $v_t$ first and then subtracting
does not work: it leaves $v_t-\Ex[X_t^2]$.

\begin{proof}[Proof of Theorem~\ref{thm:m1}]
Since $\zeta_r(a)\ge-X_r\ge-1$ we have $\eta\zeta_r(a)\ge-\eta$, and Lagrange's remainder gives
$e^{-z}\le1-z+(e^\eta/2)z^2$ for $z\ge-\eta$. With $\log(1+y)\le y$, the potential
$W_t=\sum_a\exp(-\eta\sum_{r\le t-d-1}\zeta_r(a))$ obeys
$\log(W_{t+1}/W_t)\le-\eta\langle x_t,\zeta_{t-d}\rangle
+(\eta^2e^\eta/2)\langle x_t,\zeta_{t-d}^2\rangle$. Telescoping from $W_1=K$ and using
$\log W_{T+1}\ge-\eta\sum_{r\le T-d}\zeta_r(a^\ast)$ for the comparator $a^\ast$,
\[
\sum_{r=1}^{T-d}\langle x_{r+d},\zeta_r\rangle\le\sum_{r=1}^{T-d}\zeta_r(a^\ast)
+\frac{\log K}{\eta}+\frac{\eta e^\eta}{2}\sum_{r=1}^{T-d}\langle x_{r+d},\zeta_r^2\rangle .
\]
The shift cancels between the two sides. As $x_r$ is $\mathcal F_{r-d-1}$-measurable,
$\Ex[\langle x_r,\hat c_r\rangle]=\Ex[\langle x_r,c_r\rangle]$ and $\Ex[\hat c_r(a^\ast)]=
c_r(a^\ast)$, and the $d$ rounds undelivered by $T+1$ contribute at most $d$, so
\[
\Ex[R_T]\le\frac{\log K}{\eta}
+\frac{\eta e^\eta}{2}\sum_r\Ex\bigl[\langle x_{r+d},\zeta_r^2\rangle\bigr]
+\sum_r\Ex\bigl[\langle x_r-x_{r+d},\hat c_r\rangle\bigr]+d .
\]
Lemma~\ref{lem:drift} gives $x_{r+d}\le e\,x_r$ surely, and $x_{r+d}$ is
$\mathcal F_{r-1}$-measurable, so the second sum is at most $e\sum_r\Ex[v_r-1]\le eV^-$.

For the third, write $x_{r+d}(a)=x_r(a)e^{-\eta G_r(a)}/Z'$ with
$G_r=\sum_{u=r-d}^{r-1}\hat c_u\ge0$ and normaliser $Z'\le1$, distinct from the one-step $Z_t$
of the Lemma. Then
$x_r(a)-x_{r+d}(a)\le x_r(a)(1-e^{-\eta G_r(a)})\le\eta x_r(a)G_r(a)$ for every $a$, trivially
where the left side is negative. Since $\hat c_r\ge0$ and $x_r,G_r$ are
$\mathcal F_{r-1}$-measurable, the conditional expectation is at most $\eta\langle x_r,G_r\rangle$,
and $\Ex[\langle x_r,G_r\rangle]=\sum_{u=r-d}^{r-1}\Ex[\langle x_r,\hat c_u\rangle]\le d$: for each
such $u$, $x_r$ is $\mathcal F_{r-d-1}$-measurable and $r-d-1\le u-1$, so it leaves the conditional
expectation and $\Ex[\langle x_r,\hat c_u\rangle]=\langle x_r,c_u\rangle\le1$. This is
measurability, not independence, and at $u=r-d$ the two $\sigma$-fields coincide, so the step rests
on the timing convention of Section~\ref{sec:setting}. The sum is at most $\eta dT$.

Finally $\eta\le1/(2e)$ gives $e^{1+\eta}<3.3$, so
$\Ex[R_T]\le\log K/\eta+\tfrac{\eta}{2}(3.3V^-+2dT)+d$. For $f(\eta)=A/\eta+B\eta$ restricted to
$\eta\le\eta_{\max}$ one has $\min f\le2\sqrt{AB}+A/\eta_{\max}$, in the interior and in the binding
case alike; with $A=\log K$, $B=(3.3V^-+2dT)/2$ and $\eta_{\max}=1/(e(d+1))$ this is the stated
bound.
\end{proof}

The cap on $\eta$ binds only when $d\gtrsim T/(e^2\log K)$, and not at the settings of
Appendix~\ref{app:num}: the funnel's own rates are $0.0058$, $0.0031$ and $0.0016$ at
$d=10,50,200$ against caps $0.0334$, $0.0072$ and $0.0018$, so Theorem~\ref{thm:m1} covers those
runs as they were executed. Over three seeds and $20000$ rounds the induction has room to spare,
$\max_ax_{t+1}(a)/x_t(a)$ reaching $1.006$, $1.003$ and $1.002$ against the permitted $1.100$,
$1.020$ and $1.005$.

The two bounds are numerically close on the funnel, where $\bar v\approx2$: the additive form
carries $3.3(\bar v-1)+2d$ and the interleaved one $(d+1)\bar v$, and these cross near $\bar v=2$.
The additive form gains as $\bar v$ rises toward $|\Sset|$, and it describes the algorithm that is
actually run, which is the reason to prefer it.

\end{document}